\documentclass{article}

\usepackage{subfigure}
\usepackage{booktabs} % for professional tables
\usepackage{subcaption}

\usepackage{cancel}
\usepackage{hyperref}

\usepackage[accepted]{icml2025}

\usepackage{amsmath}
\usepackage{amsthm}
\usepackage{thm-restate}
\usepackage{cancel}
\usepackage{graphicx} 
\usepackage[utf8]{inputenc} % allow utf-8 input
\usepackage[T1]{fontenc}    % use 8-bit T1 fonts
\usepackage{hyperref}       % hyperlinks
\usepackage{url}            % simple URL typesetting
\usepackage{booktabs}       % professional-quality tables\textbf{}
\usepackage{amsfonts}       % blackboard math symbols
\usepackage{nicefrac}       % compact symbols for 1/2, etc.
\usepackage{microtype}      % microtypography
\usepackage{relsize}
\usepackage{tikz}
\usepackage{wrapfig}
\usepackage{xcolor}         % colors

\usepackage[capitalize,noabbrev]{cleveref}

\theoremstyle{plain}
\newtheorem{theorem}{Theorem}[section]

\newtheorem{lemma}[theorem]{Lemma}

\theoremstyle{definition}
\newtheorem{definition}[theorem]{Definition}

\theoremstyle{remark}

\newcommand{\SSigma}{\boldsymbol{\Sigma}}
\newcommand{\vtheta}{\boldsymbol{\theta}}

\newcommand{\PP}{\boldsymbol{P}}

\newcommand{\loss}{\color{black} \mathcal{L}}

\newcommand{\rv}{\mathrm}
\newcommand{\E}{\mathbb{E}}
\newcommand{\R}{\mathbb{R}}

\newcommand{\xx}{\boldsymbol{x}}

\newcommand{\zz}{\boldsymbol{z}}

\usepackage[textsize=tiny]{todonotes}

\icmltitlerunning{Slowing Learning by Erasing Simple Features}

\begin{document}

\twocolumn[
\icmltitle{Slowing Learning by Erasing Simple Features}



\icmlsetsymbol{equal}{*}

\begin{icmlauthorlist}
\icmlauthor{Lucia Quirke}{eai}
\icmlauthor{Nora Belrose}{eai}
\end{icmlauthorlist}

\icmlaffiliation{eai}{EleutherAI}

\icmlcorrespondingauthor{Lucia Quirke}{lucia@eleuther.ai}

\icmlkeywords{Machine Learning, ICML}

\vskip 0.3in
]



\printAffiliationsAndNotice{\icmlEqualContribution} % otherwise use the standard text.

\begin{abstract}
Prior work suggests that neural networks tend to learn low-order moments of the data distribution first, before moving on to higher-order correlations. In this work, we derive a novel closed-form concept erasure method, QLEACE, which surgically removes all quadratically available information about a concept from a representation. Through comparisons with linear erasure (LEACE) and two approximate forms of quadratic erasure, we explore whether networks can still learn when low-order statistics are removed from image classification datasets. We find that while LEACE consistently slows learning, quadratic erasure can exhibit both positive and negative effects on learning speed depending on the choice of dataset, model architecture, and erasure method.

Use of QLEACE consistently slows learning in feedforward architectures, but more sophisticated architectures learn to use injected higher order Shannon information about class labels. Its approximate variants avoid injecting information, but surprisingly act as data augmentation techniques on some datasets, enhancing learning speed compared to LEACE.

% - The side effect of QLEACE makes it inappropriate for studying second order erasure with expressive model architectures, but may make it a good candidate for studying higher order information.
% - The side effect of ALF-QLEACE provides insight into the nature of image classification, hinting that the order of statistical information may vary by task and sometimes be non-monotonic.

\end{abstract}

\section{Introduction}

The success of deep learning is due in large part to the good inductive biases baked into popular neural network architectures \cite{chiang2022loss, teney2024neural}. One of these biases is the \textbf{distributional simplicity bias (DSB)}, which states that neural networks learn to exploit the lower-order statistics of the input data first-- e.g. mean and (co)variance-- before learning to use its higher-order statistics, such as  (co)skewness or (co)kurtosis.

Recent work from \citet{belrose2024neural} provides evidence for the DSB by evaluating intermediate checkpoints of a network on maximum entropy synthetic datasets which match the low-order statistics training set. They find that networks learn to perform well on the max-ent synthetic datasets early in training, then lose this ability later.

In this work, we invert the experimental setup of \citet{belrose2024neural} by training networks on datasets whose low-order statistics have been made uninformative of the class label. To remove first-order information-- that is, making all classes have the same mean-- we use the LEAst-squares Concept Erasure (LEACE) method from \citet{belrose2023leace}, which provably distorts the data only as little as is necessary. For second-order information-- differences in covariance between classes-- we derive two novel concept erasure techniques which may be of independent interest: QLEACE, and an approximate variant ALF-QLEACE.

We measure the ``difficulty'' of learning on erased data using the prequential minimum description length (MDL) framework proposed by \citet{voita2020information}. Prequential MDL is equivalent to the area under the learning curve, where the x-axis indicates the size of the training dataset, and the y-axis is the cross-entropy loss computed on a validation set after training. Additionally, we report final losses achieved on training runs which use the entire dataset.

We find that LEACE is consistently effective in making learning more difficult, causing significant increases in MDL which are consistent across network architectures. Interestingly, making the network more expressive by increasing its width beyond roughly 500 or 1K neurons does not seem to reduce the MDL, either on the control set or on the LEACE'd dataset. Since we use the maximal update parametrization \citep{yang2021tuning} to ensure hyperparameter transfer across network widths, this suggests LEACE makes learning more difficult even for \emph{infinite width} feature-learning networks \citep{yang2020feature, vyas2024feature}.

By contrast, our quadratic erasure results are much more mixed. We found that there are multiple different ways of eliminating or reducing the amount of quadratic information present in a dataset, and each of these methods seem to yield different results. While QLEACE guarantees that all classes will have equal means and covariance matrices, it can potentially inject additional information about the class label into higher-order statistics. In some cases, this leads to a phenomenon we call \emph{backfiring} where a model learns to exploit these higher-order statistics after many epochs of training, ultimately achieving \emph{lower} loss than it does on the original data.

ALF-QLEACE is meant to address this problem: since it applies the same affine transformation to every datapoint independent of its class label, it cannot cause backfiring, although it leaves a fraction of the quadratically-available information intact. We also experiment with a gradient-based approach which directly optimizes the dataset to remove quadratic information, while penalizing the distance traveled from the original data. Surprisingly, these methods seem to serve as a form of \emph{data augmentation} for convolutional networks, leading to a lower MDL than LEACE does. We conclude that known quadratic concept erasure methods are unreliable and should be used with caution.

% Overall we find large increases in MDL when erasing first- and second-order information from the training data, especially for feedforward architectures. This effect is roughly independent of the depth and width of the network, and is observed on CIFAR-10 and SVHN as well as in the larger CIFARNet dataset \cite{belrose2024neural}. Indeed, learning often only begins after many epochs of training.

\section{Erasing Moments via Optimal Transport}

\begin{figure*}
    \centering
    \includegraphics[trim=0 0 0 0, clip, width=0.95\textwidth]{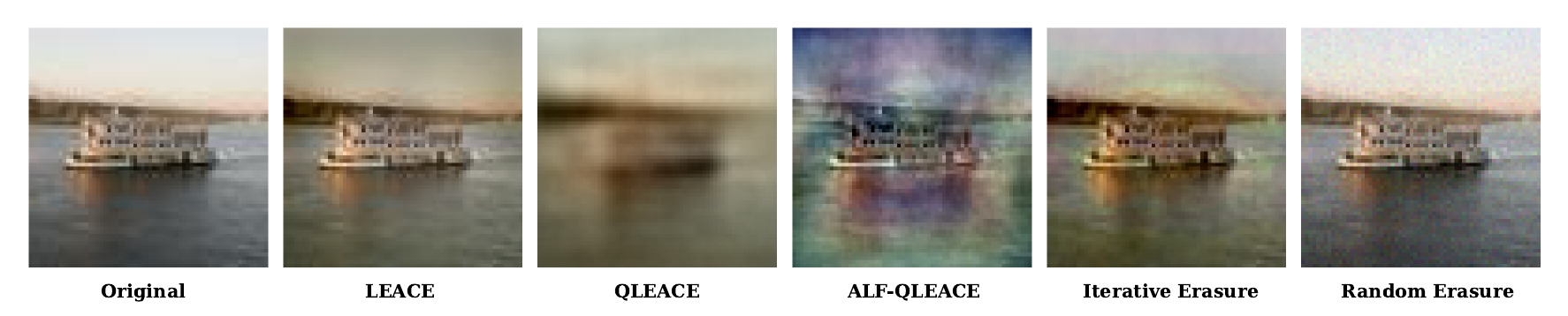}
    \caption{Ship from the CIFARNet training set edited with each eraser. LEACE and gradient-based quadratic erasure minimally affect intelligibility of the image, while QLEACE and ALF-QLEACE reduce intelligibility. The ALF-QLEACE intervention is rank $d - 15$. A random projection of equal rank is included for comparison.}
    \label{fig:eraser-comparison}
\end{figure*}

Consider a $k$-class classification task over jointly defined random vectors $\rv X$ (the input data) and $\rv Z$ (the one-hot labels), taking values in $\mathbb R^d$ and $\mathcal Z = \{(z_1, \ldots z_k) \in \mathbb \{0, 1\}^k\ \big|\ \sum_{j=1}^k z_j = 1\}$\footnote{We frequently use the integer $j \le k$ to refer to the element of $\mathcal Z$ which is $1$ at the $j^\text{th}$ index and $0$ elsewhere.} respectively, with $\E \|\rv X\| < \infty$ and each $\mathbb P(\rv Z = j) > 0$, and a predictor $\eta(\cdot; \vtheta): \mathbb R^d \to \mathbb R^k$, chosen from a function class $\mathcal V = \{\eta(\cdot; \vtheta)\ |\ \vtheta \in \Theta\}$ (presumed to contain all constant functions) so as to minimize the expectation $\E \big[ \loss(\eta(\rv X), \rv Z) \big]$ of some $\loss: \mathbb R^k \times \mathcal Z \to [0, \infty)$ in a class $\mathfrak L$ of loss functions.

\begin{definition}[Polynomial Predictor]\label{def:polynomial-predictor}
A degree $N$ \textbf{polynomial predictor} is a function $\eta : \R^d \rightarrow \R^k$, where the $j$\textsuperscript{th} component of $\eta(\xx)$ is a polynomial of degree $N$ in the components of $\xx$, and is interpreted as the likelihood that $\xx$ is a member of class $j$. In Einstein summation notation we have
\begin{equation}\label{eq:polynomial-predictor}
    \eta(\xx) = \mathbf b + \sum_{n = 1}^N \big ( \mathbf{A}^{(n)}_{i_1\ldots i_n} \xx_{i_1}\ldots \xx_{i_n} \big ),
\end{equation}
where $\mathbf b \in \R^d$ is a bias term, and each $\mathbf{A}^{(n)}$ is an order $n + 1$ coefficient tensor whose first $n$ axes are of size $d$, and whose final axis is of size $k$. For binary classification where $k = 1$, Eq.~\ref{eq:polynomial-predictor} reduces to the familiar quadratic form $\xx^T \mathbf{A} \xx + \mathbf{b}^T \xx + \mathbf{c}$ for $n = 2$, and a linear classifier $\mathbf{b}^T \xx + \mathbf{c}$ for $n = 1$.
\end{definition}

% The linear concept erasure method proposed in \citet{belrose2023leace} applies the same transformation to each data point at inference time, regardless of the class it belongs to. However, we prove in Appendix~\ref{app:oracle-leace} that it is strictly more surgical to apply a \emph{different} transformation for each class, and derive the optimal class-dependent concept erasure function for the linear case.

Once we allow the erasure function to depend on the class label, we can immediately draw a connection with the Monge formulation of the \textbf{optimal transport} problem between probability distributions. First consider the \textbf{information-theoretic} erasure task: given $k$ class-conditional distributions $\{ \mathbb{P}(\xx | \rv Z = \zz_1), \ldots, \mathbb{P}(\xx | \rv Z = \zz_k) \}$ and a cost function ${c:\mathcal{X}\times\mathcal{Z}\rightarrow\mathbb{R}}$, we seek a \textbf{barycenter} distribution $\overline{\mathbb{P}}(\xx)$ and a set of transport maps $\{ T_1(\xx), \ldots, T_k(\xx) \}$ minimizing expected cost while transporting each class to $\overline{\mathbb{P}}(\xx)$.
\begin{equation}
    C(\mathbb{P},\mathbb{Q})\stackrel{\text{def}}{=}\inf_{T_{\#}\mathbb{P}=\mathbb{Q}}\ \int_{\mathcal{X}} c\big(x,T(x)\big) d\mathbb{P}(x),
\label{ot-primal-form-monge}
\end{equation}
where the minimum is taken over measurable functions $T:\mathcal{X}\rightarrow\mathcal{Y}$ that map $\mathbb{P}$ to $\mathbb{Q}$. It's easy to see that an exact solution to this task would eliminate all Shannon information between $\rv X$ and $\rv Z$, preventing unrestricted nonlinear adversaries from predicting $\rv Z$ from $\rv X$ better than chance.\footnote{Strictly speaking, the Monge formulation of OT does not always have a solution: for example, it is impossible for any deterministic map to transport the Dirac delta distribution to a Gaussian. The Kantorovich formulation, which allows for stochastic transport maps, always has a solution, but in this work we assume the Monge problem is feasible on real-world datasets.}

\subsection{Extension to Polynomial Classifiers}
\label{sec:polynomials}

Below, we extend \citet[Theorem 3.1]{belrose2023leace} to general polynomial predictors.

\begin{restatable}{theorem}{sufficiency}\label{sufficiency}
    Suppose $\loss$ is convex in $\eta(\xx)$. Then if for each class $\zz \in \mathcal Z$ and each order $n \in 1\ldots N$, the tensor of class-conditional moments $\E\big[\rv X_{i_1}\ldots \rv X_{i_n} | \rv Z = \zz\big]$ is equal to the unconditional moment tensor $\E\big[\rv X_{i_1}\ldots \rv X_{i_n} \big]$, the trivially attainable loss cannot be improved upon.
\end{restatable}

\begin{proof}
    See Appendix~\ref{sufficiency-proof}.
\end{proof}

We also extend \citet[Theorem 3.2]{belrose2023leace} to the polynomial case:

\begin{restatable}{theorem}{necessity}\label{necessity}
    Suppose $\loss$ has bounded partial derivatives, which when off-category never vanish and do not depend on the category, i.e.
        ${\partial \loss(\eta, z_1)}/{\partial \eta_i} = {\partial \loss(\eta, z_2)}/{\partial \eta_i} \neq 0$
    for all categories $z_1, z_2 \neq i$. If $\E \big[\loss(\eta, \rv Z)\big]$ is minimized among degree $N$ polynomial predictors by the constant predictor with $\forall n \in 1\ldots N : \mathbf A^{(n)} = \mathbf{0}$, then for each order $n \in 1\ldots N$, the tensor of class-conditional moments $\E\big[\rv X_{i_1}\ldots \rv X_{i_n} | \rv Z = \zz\big]$ is equal to the unconditional moment tensor $\E\big[\rv X_{i_1}\ldots \rv X_{i_n} \big]$.
\end{restatable}

\begin{proof}
    See Appendix~\ref{necessity-proof}.
\end{proof}

As noted in \citet[Theorem 3.3]{belrose2023leace}, this theorem applies to the popular categorical cross entropy loss $\loss(\eta, z) = - \log \frac {\exp(\eta_z)} {\sum_{i=1}^k \exp(\eta_i)}$.

% Given random vectors $\rv X$ and $\rv Z$ with joint density $\mathbb{P}(\xx, \zz) : $
%\textbf{Information-theoretic erasure.} Given a random vector $\rv X$ and discrete random variable $\rv Z$\footnote{Much of what follows appears to be extensible to the case where $\rv Z$ is a continuous random variable or vector, at the cost of considerable additional technical complexity. We focus on the discrete case for simplicity and because it appears to be more useful in practice.} with positive mutual information $I(\rv X; \rv Z)$, the theoretical ideal of concept erasure is to find the ``nearest'' random vector $\rv X'$ to $\rv X$ that contains no information about $\rv Z$. Equivalently, we would like the density of $\rv X'$ conditional on $\rv Z$ to be precisely equal to the marginal density:
%\begin{align}
%    I(\rv X; \rv Z) = \sum_{\zz \in \mathcal Z} \int_{\mathcal X} \big[ \log \mathbb P(\xx| \rv Z = \zz) - \log \mathbb P(\xx) \big] d\mathbb{P}(\xx,)  = 0
%\end{align}
%\begin{align}
%    H(\rv X' | \rv ) =  \\
%    \E_{\rv X'} \big[ -\log \mathbb P(\xx | \rv Z = \zz) \big] = \E_{\rv X'} \big[ -\log \mathbb P(\xx) \big].
%\end{align}
%\begin{equation}
%    \forall \xx \in \mathcal X, \zz \in \mathcal Z : \mathbb P(\xx | \rv Z = \zz) = \mathbb P(\xx).
%\end{equation}

\subsection{Quadratic LEACE}
\label{sec:qleace}

The theorems of Sec.~\ref{sec:polynomials} imply that our erasure function must ensure the class-conditional mean and covariance equal the unconditional mean and covariance, but they do not specify what these unconditional moments should be. We are therefore free to select the target mean and covariance matrix $\mathbf{\Sigma}^*$ so as to minimize the expected edit magnitude $\E \| \rv X' - \rv X \|^2_2$. Using results from optimal transport theory, we derive the optimal unconditional mean and covariance matrix, as well as the optimal class-dependent quadratic concept erasure function.

\begin{lemma}[Gaussian Wasserstein Barycenter]\label{thm:gaussian-barycenter}
    Let $\nu_1, \ldots, \nu_k$ be Gaussian distributions with means $m_1, \ldots, m_k$ and non-singular covariance matrices $\mathbf{\Sigma}_1, \ldots, \mathbf{\Sigma}_k$, along with non-negative weights $\lambda_1, \ldots, \lambda_k$. Let $\mathcal{P}_2(\mathbb{R}^d)$ be the set of probability measures over $\mathbb{R}^d$ with finite second moment. Then
    \begin{equation}\label{eq:gaussian-objective}
        \mathcal{N}(\bar{\mathbf{m}}, \bar{\mathbf{\Sigma}}) = \mathop{\mathrm{argmin\:}}_{\mu \in \mathcal{P}_2(\mathbb{R}^d)} \sum_{i = 1}^k \lambda_i \mathcal{W}_2^2(\nu_i, \mu),
    \end{equation}
    where $\bar{\mathbf{m}} = \sum_{i = 1}^k \lambda_i \bar{\mathbf{m}}_i$ and $\bar{\mathbf{\Sigma}}$ is the unique positive definite $\mathbf{S}$ satisfying the equation
    \begin{equation}\label{eq:barycenter}
        \mathbf{S} = \sum_{i = 1}^k \lambda_i \big( \mathbf{S}^{1/2} \mathbf{\Sigma}_i \mathbf{S}^{1/2} \big)^{1/2}.
    \end{equation}
    %We may call $\mathcal{N}(\bar m, \bar{\mathbf{\Sigma}})$ the \textbf{2-Wasserstein barycenter} of $\nu_1, \ldots, \nu_k$.
\end{lemma}
\begin{proof}
    See \citet[pg. 246]{ruschendorf2002n}.
\end{proof}

\begin{lemma}[2-Wasserstein Lower Bound]\label{thm:lower-bound}
    Let $\mathcal{P}_2(\mathbb{R}^d)$ be the set of probability measures over $\mathbb{R}^d$ with finite second moment. Let $P, Q \in \mathcal{P}_2(\mathbb{R}^d)$ have means $\mathbf{m}_P, \mathbf{m}_Q$ and covariance matrices $\mathbf{\Sigma}_P, \mathbf{\Sigma}_Q$, with $\mathbf{\Sigma}_P$ full rank. Then
    \begin{align}
        \mathcal{W}_2^2(P, Q) \geq \mathcal{W}_2^2(\mathcal{N}(\mathbf{m}_P, \mathbf{\Sigma}_P), \mathcal{N}(\mathbf{m}_Q, \mathbf{\Sigma}_Q)),
    \end{align}
    with equality if and only if the map $T(\xx) = \mathbf{A}(\xx - \mathbf{m}_P) + \mathbf{m}_Q$ transports $P$ to $Q$, where
    \begin{equation}\label{eq:}
        \mathbf{A} = \mathbf{\Sigma}_P^{-1/2} \big ( \mathbf{\Sigma}_P^{1/2} \mathbf{\Sigma}_Q \mathbf{\Sigma}_P^{1/2} \big )^{1/2} \mathbf{\Sigma}_P^{-1/2}.
    \end{equation}
\end{lemma}
\begin{proof}
    See \citet{cuesta1996lower}.
\end{proof}

\begin{figure*}[t]
    \centering
    \includegraphics[width=0.9\textwidth]{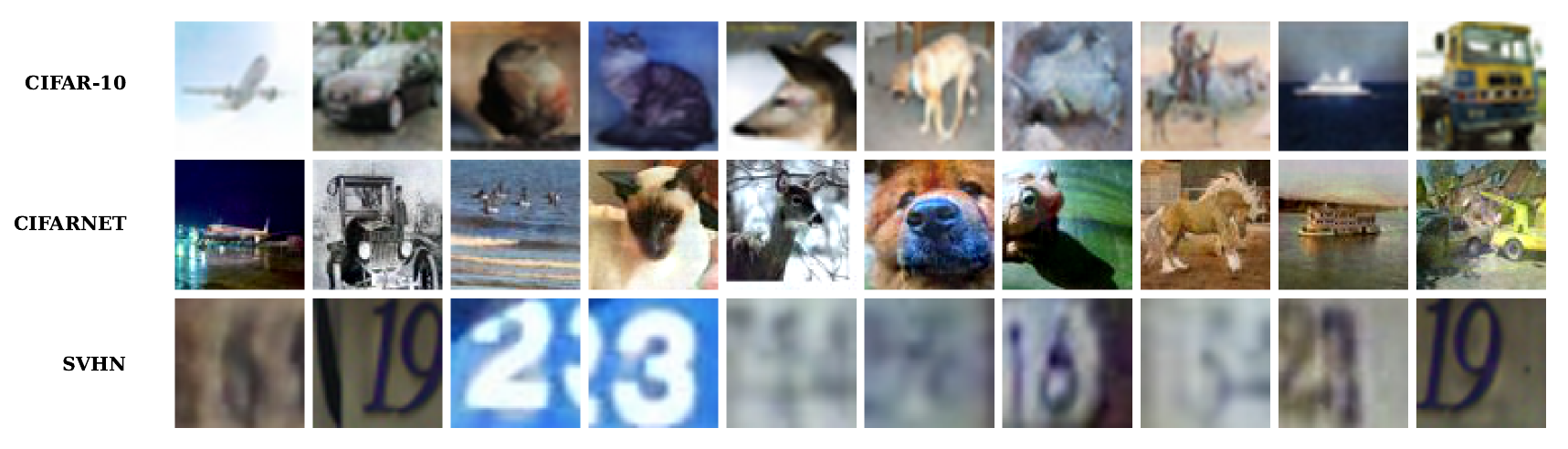}
    \caption{Images can be easily identified after gradient-based erasure. Non-cherrypicked examples from CIFAR-10, CIFARNet, and SVHN.}
    \label{fig:iterative}
\end{figure*}

\begin{theorem}[Quadratic LEACE]\label{thm:qleace}
Let $\nu_1, \ldots, \nu_k$ be distributions on $\R^d$ with means $\mathbf{m}_1, \ldots, \mathbf{m}_k$, full rank covariance matrices $\mathbf{\Sigma}_1, \ldots, \mathbf{\Sigma}_k$, and weights $\lambda_1, \ldots, \lambda_k \geq 0$. Let $\mathcal{P}(\R^d; \mathbf{m}, \mathbf \Sigma)$ be the set of distributions on $\R^d$ with mean $\mathbf{m}$ and covariance matrix $\mathbf \Sigma$. Then the objective
\begin{equation}\label{eq:recoloring-objective}
    \mathop{\mathrm{argmin\:}}_{(\mathbf{m}, \mathbf{\Sigma}) \in \R^d \times \mathcal{S}_{++}^d} \sum_{i = 1}^k \lambda_i \min_{\mu_i \in \mathcal{P}(\mathbb{R}^d; \mathbf{m}, \mathbf \Sigma)} \mathcal{W}_2^2(\nu_i, \mu_i)
\end{equation}
is minimized by the mean and covariance matrix of the barycenter of $\mathcal{N}(\mathbf{m}_1, \mathbf{\Sigma}_1), \ldots, \mathcal{N}(\mathbf{m}_k, \mathbf{\Sigma}_k)$, as defined in Lemma~\ref{thm:gaussian-barycenter}.

Further, for each $i$, the optimal transport map from $\nu_i$ to the optimal $\mu_i$ is the optimal transport map from $\mathcal{N}(\mathbf{m}_i, \mathbf{\Sigma}_i)$ to $\mathcal{N}(\bar{\mathbf{m}}, \bar{\mathbf{\Sigma})}$, as defined in Lemma~\ref{thm:lower-bound}.
\end{theorem}
\begin{proof}
We first consider the inner optimization problem for any fixed $(\mathbf{m}, \mathbf{\Sigma})$ chosen for the outer optimization problem.

Our only constraint for each $u_i$ is that $u_i \in \mathcal{P}(\mathbb{R}^d; \mathbf{m}, \mathbf \Sigma)$. Thus $\mu_i \sim \mathbf{A}_i(\nu_i - \mathbf{m}_i) + \mathbf{m}$ is feasible, with $\mathbf{A}_i = \mathbf{\Sigma}_i^{-1/2} \big ( \mathbf{\Sigma}_i^{1/2} \mathbf{\Sigma} \mathbf{\Sigma}_i^{1/2} \big )^{1/2} \mathbf{\Sigma}_i^{-1/2}$. Further, by Lemma~\ref{thm:lower-bound}, this choice achieves the minimum possible cost in the feasible set $\mathcal{P}(\mathbb{R}^d; \mathbf{m}, \mathbf \Sigma)$, equal to $\mathcal{W}_2^2(\mathcal{N}(\mathbf{m}_i, \mathbf{\Sigma}_i), \mathcal{N}(\mathbf{m}, \mathbf{\Sigma}))$.

We draw the following two conclusions.

First, for the optimal choice of $(\mathbf{m}, \mathbf{\Sigma})$, defined as $(\bar{\mathbf{m}}, \bar{\mathbf{\Sigma}})$, the optimal transport map for each $i$ is
\begin{equation}\label{eq:optimal-T-i}
T_i(\mathbf{x}) = \mathbf{\Sigma}_i^{-1/2} \big ( \mathbf{\Sigma}_i^{1/2} \bar{\mathbf{\Sigma}} \mathbf{\Sigma}_i^{1/2} \big )^{1/2} \mathbf{\Sigma}_i^{-1/2}(\mathbf{x} - \mathbf{m}_i) + \bar{\mathbf{m}}
\end{equation}

Second, we can rewrite our objective as 
\begin{equation}\label{eq:rewritten-recoloring}
    \mathop{\mathrm{argmin\:}}_{(\mathbf{m}, \mathbf{\Sigma}) \in \R^d \times \mathcal{S}_{++}^d} \sum_{i = 1}^k \lambda_i \mathcal{W}_2^2(\mathcal{N}(\mathbf{m}_i, \mathbf{\Sigma}_i), \mathcal{N}(\mathbf{m}, \mathbf{\Sigma})).
\end{equation}

This is identical to the Gaussian barycenter objective from Lemma~\ref{thm:gaussian-barycenter}, except that we have restricted the feasible set to Gaussian distributions. Since the barycenter of Gaussian distributions is always Gaussian, this restriction makes no difference to the solution. Therefore the barycenter of $\mathcal{N}(\mathbf{m}_1, \mathbf{\Sigma}_1), \ldots, \mathcal{N}(\mathbf{m}_k, \mathbf{\Sigma}_k)$ will have mean and covariance $(\bar{\mathbf{m}}, \bar{\mathbf{\Sigma}})$ that minimize Eq.~\ref{eq:rewritten-recoloring}, thus completing the proof.
\end{proof}

To solve for the unconditional covariance matrix, we use the fixed-point algorithm derived in \citet{alvarez2016fixed}. Then the optimal transport map can be computed in closed form as described above.

\begin{figure*}[t]
    \centering
    \includegraphics[width=0.8\textwidth]{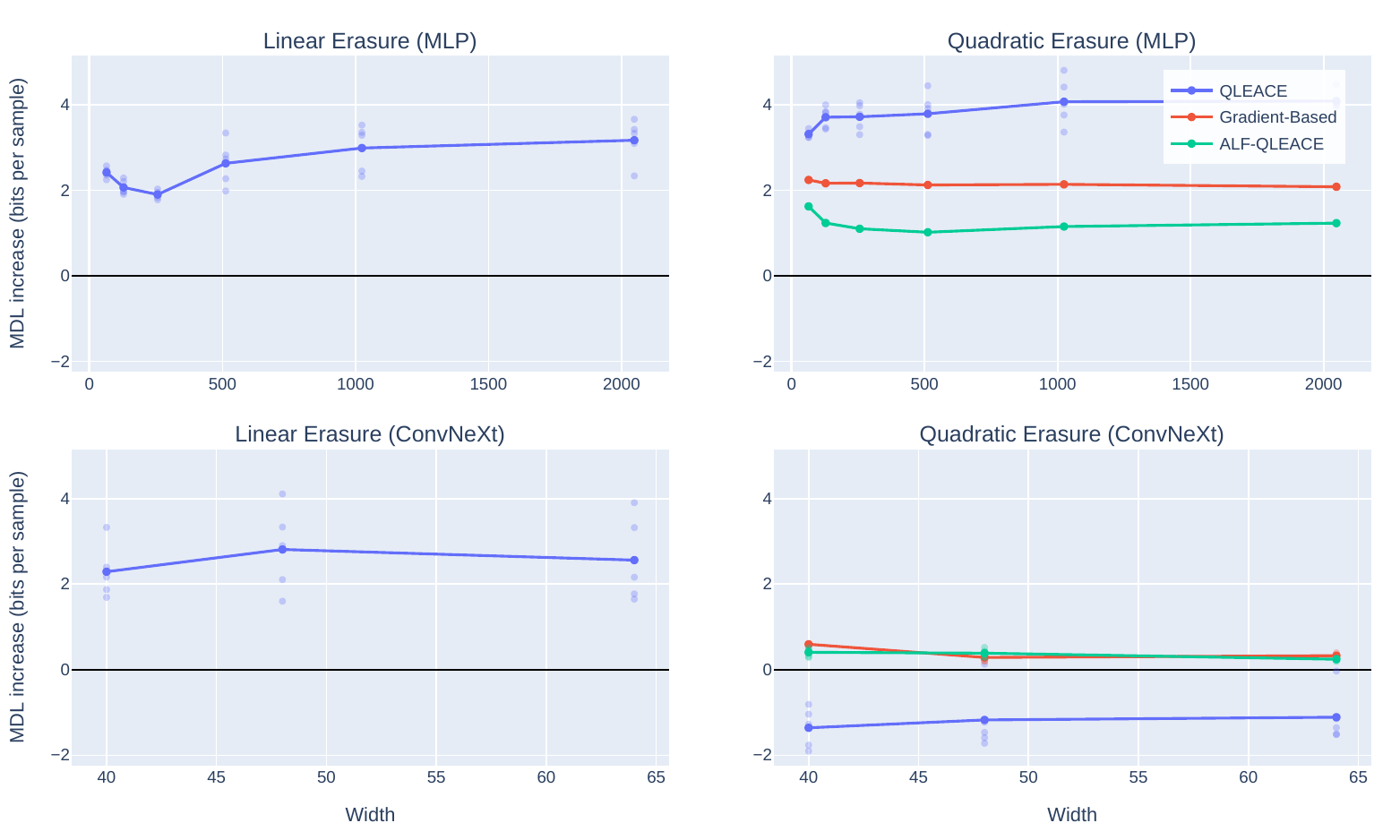}
    \caption{Increase in MDL from the erasure of the CIFAR-10 dataset over 5 random seeds for ReLU MLPs and ConvNeXt V2s of various widths. While LEACE slows learning similarly in both architectures, ConvNeXts are less affected by quadratic erasure and exhibit a backfiring effect on data modified with QLEACE, resulting in improved performance relative to unerased data.}
    \label{fig:cifar10_mdl_quadratic}
\end{figure*}

\subsection{Approximate Label-Free Quadratic LEACE}

While naïve quadratic LEACE removes quadratically available information, it also potentially injects Shannon information about the class label. Consider a dataset where each sample is a $d$-dimensional data point from one of $k$ classes. The data naturally lies within a bounded hypercube $[0, 1]^d$ or $[0, 255]^d$ defined by the valid range of each dimension. By applying an erasure function specific to the class label of each sample, QLEACE effectively transforms the hypercube into $k$ characteristic hyperparallelipipeds. Neural networks of sufficient expressivity may use this information to classify data.

Approximate Label-Free QLEACE removes some quadratically available information without injecting label information by applying the same transformation to each data point. The transformation is composed of a LEACE transformation to remove linearly available information, and an ALF-QLEACE transformation to remove some quadratically available information information.

Let $\SSigma_k$ denote the class-conditional covariance matrix for class $k$, and let $\bar{\SSigma}$ be the average of all class-conditional covariance matrices. Perfect quadratic guardedness would imply that
\begin{equation}\label{eq:perfect-guardedness}
    \SSigma_k - \bar{\SSigma} = \mathbf{0}_{d \times d}, \quad \forall k \in K
\end{equation}
This is trivially satisfied if $\PP$ is the zero matrix, but this would remove \emph{all} information contained in the representation. We need a higher-rank transformation.

The rank $d - k$ ALF-QLEACE transformation is a projection matrix $\PP$ that satisfies
\begin{align}\label{eq:ols-problem}
\mathop{\mathrm{argmin}}_{\substack{\PP}} \max_{k \sim K} \| \PP \SSigma_k \PP^T - \PP \bar{\SSigma} \PP^T \|^2_2
\end{align}
The \href{https://en.wikipedia.org/wiki/Min-max_theorem}{min-max theorem} implies that the rank $d - r$ orthogonal projection of a matrix with respect to both Frobenius and spectral norms is given by truncating its singular value decomposition to the $n$ largest singular values. 

Therefore, the rank 1 projection matrix that best approximates the difference matrix for a class $D_k = \Sigma_k - \bar{\SSigma}$ is given by the largest singular value of its singular value decomposition, and the optimal rank $d - 1$ projection matrix $P$ which minimizes that $D_k$ projects onto the orthogonal complement of the direction that contributes most to this maximum singular value - specifically, either singular vector corresponding to the largest singular value of $D_k$ for the class $k$ that maximizes the norm.

The rank $d-1$ projection preserves as much of the original data structure as possible while still effectively removing the strongest quadratic correlations. The projection does not increase variance in any direction and is therefore guaranteed to not increase the norm for another class. 

The efficacy of the projection may be improved by iteratively adding up to $d$ rank 1 erasures, each calculated on the updated class and global covariances matrices $P \SSigma_k P$ and $P \bar{\SSigma} P$.

% Transforming to a specified covariance matrix is sometimes called a 'coloring' operation , vs. whitening

\subsection{Gradient-based Moment Erasure}

We also experiment with a much more ``direct'' approach to erasing low-order class information from a dataset, using gradient-based optimization. Specifically, we initialize our synthetic dataset using the original data, and use L-BFGS \cite{liu1989limited} to optimize a loss function containing three terms: the average squared distance between each class-conditional mean and the unconditional mean, the average squared distance between each class-conditional covariance matrix and the unconditional covariance matrix, and the squared distance between the original images and their modified counterparts. We tune the weights of the three loss terms to ensure that the synthetic dataset has almost no quadratically-available information about the class label, while keeping fairly close to the original images. We also re-parameterize the images using the sigmoid function to ensure pixel values stay within the admissible range $[0, 255]$.

This approach has the benefit of not causing backfiring, and being mathematically straightforward. Unlike the other methods, however, it does not produce an ``eraser'' transformation that can be re-used on new datapoints.

% TODO add other eraser comparisons to appendix 

\section{Experiments}

We investigate the effects of erasing simple features from three popular image classification datasets: CIFAR-10 \citep{krizhevsky2009learning}, CIFARNet \citep{belrose2024neural}, and Street View House Numbers (SVHN) \citep{netzer2011reading}. For data augmentation we apply random crops and horizontal flips, both linear transformations which cannot recover erased information. For each dataset we examine the effects on learning of LEACE, quadratic erasure, and two approximate variants of quadratic erasure.

We use the maximal update parametrization ($\mu$P) to transfer the optimal hyperparameters from narrower neural networks to wider ones \citep{yang2021tuning}. When scaling our MLPs to different depths, we adapt the experimental learning rate scaling exponent of $3/2$ proposed in \citet{jelassi2023depthdependencemuplearning}, and use a more conservative scaling exponent of $1/2$ for all other architectures,
$$lr=lr_{base} \sqrt{\frac{L_{base}}{L}}$$
where $L$ is the number of layers.

We also use the schedule-free AdamW optimizer proposed by \citet{defazio2024road}, which allows us to train our networks for as long as is necessary to achieve convergence without fixing a learning rate decay schedule in advance. 

To investigate how dataset choice mediates the effects of feature erasure we train MLPs and parameter-matched LeNets on the CIFAR-10, CIFARNet, and SVHN datasets. For each combination of dataset, eraser, and architecture we track training dynamics via cross-entropy loss curves for various widths and depths and compute the prequential MDL, averaging results across five random seeds.

To investigate the effects of erasure on state-of-the-art image classification architectures, we train Swin Transformer V2 \citep{liu2022swin} and ConvNeXt V2 \citep{woo2023convnext} models on the CIFAR-10 dataset. We scale the models width- and depth-wise and collect average prequential MDLs (scaling details are available in Appendix~\ref{appendix:scaling}).

Finally, we examine the effect of dataset z-score normalization on erasures.

\section{Results}

\paragraph{Linear Erasure (LEACE)}

Applying LEACE to datasets consistently increases MDL and final loss across all models and tasks. This effect is roughly equal between architectures, and varies little across model widths and depths.

\paragraph{QLEACE}

QLEACE generally produces larger MDL and loss increases than linear erasure in the feedforward networks. In state-of-the-art Swin and ConvNeXt architectures, however, the impact is mixed. QLEACE slows learning more than any other erasure methods in early epochs, with early loss surpassing models trained on linearly erased datasets (Figures~\ref{fig:convnext_loss} and ~\ref{fig:swin_loss}). However, after approximately 16 epochs a sharp transition occurs: loss plummets rapidly, allowing these models to ultimately achieve lower final loss than models trained on unedited data.

We call this phenomenon ``backfiring.'' For Swin architectures, the prolonged high-loss phase during early training ultimately produces higher prequential MDL than unedited baselines – closely matching the MDL of LEACE-edited data. ConvNeXt models, by contrast, achieve lower overall MDL with QLEACE than with any other erasure method, even though the backfiring occurs late in training.

Additionally, backfiring occurs in all two- and three-layer MLPs on the CIFARNet dataset, suggesting that even simple networks are able to access the injected higher order information in some datasets.
% TODO:  We speculate that the higher resolution images may have more characteristic hyperparallelipipeds.

While QLEACE produces similar increases in MDL across feedforward architectures, ReLU-based MLPs tend to be more affected at greater depths, and GeLU-based MLPs and SwiGLUs in shallow settings (Figure~\ref{fig:cifar10_ffn_mdl_quadratic}).

\paragraph{Gradient-Based Quadratic Erasure}

The gradient-based approach to quadratic erasure increases MDL across all architectures without producing backfiring. However, its efficacy relative to QLEACE varies, notably yielding substantially reduced MDL in feedforward networks - contexts where quadratic erasure does not produce backfiring. This disparity suggests that partial quadratic information may remain accessible even after aggressive gradient-based editing. These findings highlight potential limitations in gradient-based methods.

\paragraph{Approximate Quadratic Erasure (ALF-QLEACE)}

Datasets modified by ALF-QLEACE show a small increase in prequential MDL across all architectures. Surprisingly, the effect relative to LEACE is smaller in most cases, despite ALF-QLEACE applying LEACE as an initial step. This indicates that, by ablating directions along which the classes have very different variances, ALF-QLEACE is actually making the data \emph{easier} to learn.

We speculated that the smaller effect of ALF-QLEACE on MDL is caused by a data augmentation effect, where its partial normalization of the data covariance matrix by removing the directions of highest variance improves learning. We examined whether z-score normalizing the CIFAR-10 dataset before erasure affected the relative increase in MDL, with ambiguous results (Table \ref{tab:erasure_comparison}).

\section{Conclusion}

Due to the backfiring phenomena we observed in all of our quadratic concept erasure methods, we urge practitioners to exercise caution when applying them in practice. By contrast, our results suggest that LEACE is a reliable method for making features less salient and more difficult to learn.

% We focused on an image classification setting in this paper largely as a matter of convenience, but we think that the most compelling practical applications of concept erasure are likely to 

\section*{Contributions and acknowledgements}

Nora derived QLEACE and ALF-QLEACE, and wrote the mathematical parts of the paper. Lucia ran all of the experiments and analyzed the results.

We are grateful to Coreweave for providing computing resources for this project. Nora and Lucia are funded by \href{https://www.openphilanthropy.org/grants/eleuther-ai-interpretability-research/}{a grant} from Open Philanthropy.

\section*{Impact Statement}

As discussed in \citet{belrose2023leace}, concept erasure methods can be used to enhance the fairness of deep learning models by making sensitive information, like gender or race, less salient. On the other hand, methods closely related to concept erasure can also be used to bypass safety training in large language models \citep{marshall2024refusal}. Overall, we judge that the ability to control what types of information are used by deep neural networks is beneficial for their interpretability and steerability.


\bibliography{citations}
\bibliographystyle{icml2025}

%%%%%%%%%%%%%%%%%%%%%%%%%%%%%%%%%%%%%%%%%%%%%%%%%%%%%%%%%%%%%%%%%%%%%%%%%%%%%%%
%%%%%%%%%%%%%%%%%%%%%%%%%%%%%%%%%%%%%%%%%%%%%%%%%%%%%%%%%%%%%%%%%%%%%%%%%%%%%%%
% APPENDIX
%%%%%%%%%%%%%%%%%%%%%%%%%%%%%%%%%%%%%%%%%%%%%%%%%%%%%%%%%%%%%%%%%%%%%%%%%%%%%%%
%%%%%%%%%%%%%%%%%%%%%%%%%%%%%%%%%%%%%%%%%%%%%%%%%%%%%%%%%%%%%%%%%%%%%%%%%%%%%%%
\newpage
\appendix
\onecolumn
\section{Concept erasure proofs}

\subsection{Sufficiency}
\label{sufficiency-proof}

\sufficiency*

\begin{proof}
Let $\eta$ be an arbitrary polynomial predictor of degree $N$. We can derive a lower bound on the expected loss of $\eta$ evaluated on $\rv X$ as follows.
\begin{align*}
    \E_{(\xx, \zz)} \Big[\loss(\eta(\xx), \zz)\Big]
    = \mathbb E_{\zz} & \Big[ \E_{\xx} \Big[\loss(\eta(\xx), \zz) \big| \zz \Big] \Big] \\
    \ge \mathbb E_{\zz} & \Big[ \loss\Big( \E_{\xx} \big[ \eta(\xx) \big| \zz \big], \zz \Big) \Big] \tag{Jensen's inequality} \\
    = \mathbb E_{\zz} & \Big[ \loss\Big( \mathbf b + \sum_{n = 1}^{N} \mathbf A^{(n)}_{i_1 \ldots i_n} \E \big[ \rv X_{i_1}\ldots \rv X_{i_n} \big| \zz \big], \zz \Big) \Big] \tag{Def.~\ref{def:polynomial-predictor} and linearity} \\
    = \mathbb E_{\zz} &\Big[ \loss\Big( \mathbf b + \sum_{n = 1}^{N} \mathbf A^{(n)}_{i_1 \ldots i_n} \E \big[ \rv X_{i_1}\ldots \rv X_{i_n} \big], \zz \Big) \Big] \tag{by assumption} \\
    = \mathbb E_{\zz} &\Big[ \loss\Big( \eta'(\xx), \zz \Big) \Big] \tag{definition of $\eta'$}.
\end{align*}

The trivially attainable loss is the lowest loss achievable by a constant function. Since $\eta'$ is a constant function and $\eta$ was arbitrary, no $\eta$ can improve upon the trivially attainable loss.
\end{proof}

\subsection{Necessity}
\label{necessity-proof}

\necessity*

\begin{proof}
    Fix some class $j \in 1 \ldots k$. The first-order optimality condition on $\mathbf b_j$ and $\mathbf A^{(n)}_{\ldots j}$ yields\footnote{We are using NumPy-style indexing here, so that $\mathbf A^{(n)}_{\ldots j}$ refers to the slice of $\mathbf A^{(n)}$ where the final axis index is set to $j$. This yields a ``square'' order-$n$ tensor whose axes are all of size $d$.}
    \begin{equation}\label{opt}
        \E_{(\xx, \zz)} \Bigg[ \frac {\partial \loss(\eta, \zz)} {\partial \eta_j} \cdot \frac {\partial \eta_j} {\partial \mathbf b_j} \Bigg] = 0 \quad \text{and} \quad \forall n \in 1 \ldots N : \E_{(\xx, \zz)} \Bigg[ \frac {\partial \loss(\eta, \zz)} {\partial \eta_j} \cdot \frac {\partial \eta_j} {\partial \mathbf A^{(n)}_{\ldots j}} \Bigg] = \mathbf{0}.
    \end{equation}
    %where we have used the boundedness of $\loss$'s partial derivative and the finite first moment of $\frac {\partial \eta_i} {\partial b_i} = 1$ and $\frac {\partial \eta_i} {\partial \mathbf{W_i}} = \rv X$ to justify (via the Dominated Convergence Theorem) interchanging the derivative with the expectation.

    Since $\eta$ is constant over all values of $\rv X$, and $\frac {\partial \eta_j} {\partial \mathbf b_j} = 1$, the first equation in (\ref{opt}) reduces to:    
    \begin{equation} \label{optBreduced}
        \mathbb P(\rv Z = j) \frac {\partial \loss(\eta, j)} {\partial \eta_j} +
            \mathbb P(\rv Z \neq j) \frac {\partial \loss(\eta, \neq j)} {\partial \eta_j} = 0,
    \end{equation}
    where $\frac {\partial \loss(\eta, \neq i)} {\partial \eta_i}$ is an abuse of notation denoting the off-category partial derivative, emphasizing its independence of the category $\rv Z$. For each $n$, the constancy of $\eta$ and the fact that $\frac {\partial \eta_i} {\partial \mathbf A^{(n)}_{\ldots j}} = \xx_{i_1} \ldots \xx_{i_n}$ reduces the second part of (\ref{opt}) to:
    \begin{equation} \label{optWreduced}
        \mathbb P(\rv Z = j) \frac {\partial \loss(\eta, j)} {\partial \eta_j} \cdot \E \big[ \rv X_{i_1}\ldots \rv X_{i_n} \big| \rv Z = j \big] + \\
        \mathbb P(\rv Z \neq j) \frac {\partial \loss(\eta, \neq j)} {\partial \eta_j} \cdot \E \big[ \rv X_{i_1}\ldots \rv X_{i_n} \big| \rv Z \neq j \big]
        = \mathbf{0}.
    \end{equation}
    Solving for $\mathbb P(\rv Z = j) \frac {\partial \loss(\eta, j)} {\partial \eta_j}$ in (\ref{optBreduced}) and substituting in (\ref{optWreduced}) gives us:
    \begin{equation}\label{eq:moment-constraint}
        \forall n \in 1 \ldots N : \mathbb P(\rv Z \neq j) \frac {\partial \loss(\eta, \neq j)} {\partial \eta_j} \cdot \Bigg(
                    \E \big[ \rv X_{i_1}\ldots \rv X_{i_n} \big| \rv Z \neq j \big] -
                    \E \big[ \rv X_{i_1}\ldots \rv X_{i_n} \big| \rv Z = j \big]
        \Bigg) = \mathbf{0}.
    \end{equation}
    If $\mathbb P(\rv Z \neq j) = 0$, then all class-conditional moments are \emph{trivially} equal to the unconditional moments. Otherwise, we can divide both sides by $\mathbb P(\rv Z \neq j)$ and by $\frac {\partial \loss(\eta, \neq j)} {\partial \eta_j}$ (using the non-vanishingness of the off-category partial derivative), yielding for each order $n \in 1 \ldots N$:  %equivalence of $\E \big[ \rv X \big| \rv Z = i \big]$ to $\E \big[ \rv X \big| \rv Z \neq i \big]$, and hence to the unconditional mean $\E \big[ \rv X \big]$.
    \begin{align*}
        \E \big[ \rv X_{i_1}\ldots \rv X_{i_n} \big| \rv Z = j \big] = \E \big[ \rv X_{i_1}\ldots \rv X_{i_n} \big| \rv Z \neq j \big] = \E \big[ \rv X_{i_1}\ldots \rv X_{i_n} \big],
    \end{align*}
    thereby completing the proof.
\end{proof}

\section{Model scaling details} \label{appendix:scaling}

ConvNeXt and Swin models are produced using width and depth parameters, which represent base network parameters and are scaled in later stages according to a conventional structure.

In the ConvNeXt models, the base depth $D$ is scaled in the third stage to produce the depths $[D,D,3D,D]$, and width doubles in each successive stage, following standard practice for progressive feature map expansion. The Swin models follow the same general scaling structure, with the number of heads also scaling in the third stage according to $[D,D,3D]$.

Our ConvNeXt and Swin implementations use a reduced patch size of 1 and 2 respectively to accommodate the datasets' small image sizes.

\newpage
\section{Detailed results}

\begin{figure}[h]
    \centering
    \includegraphics[width=\textwidth]{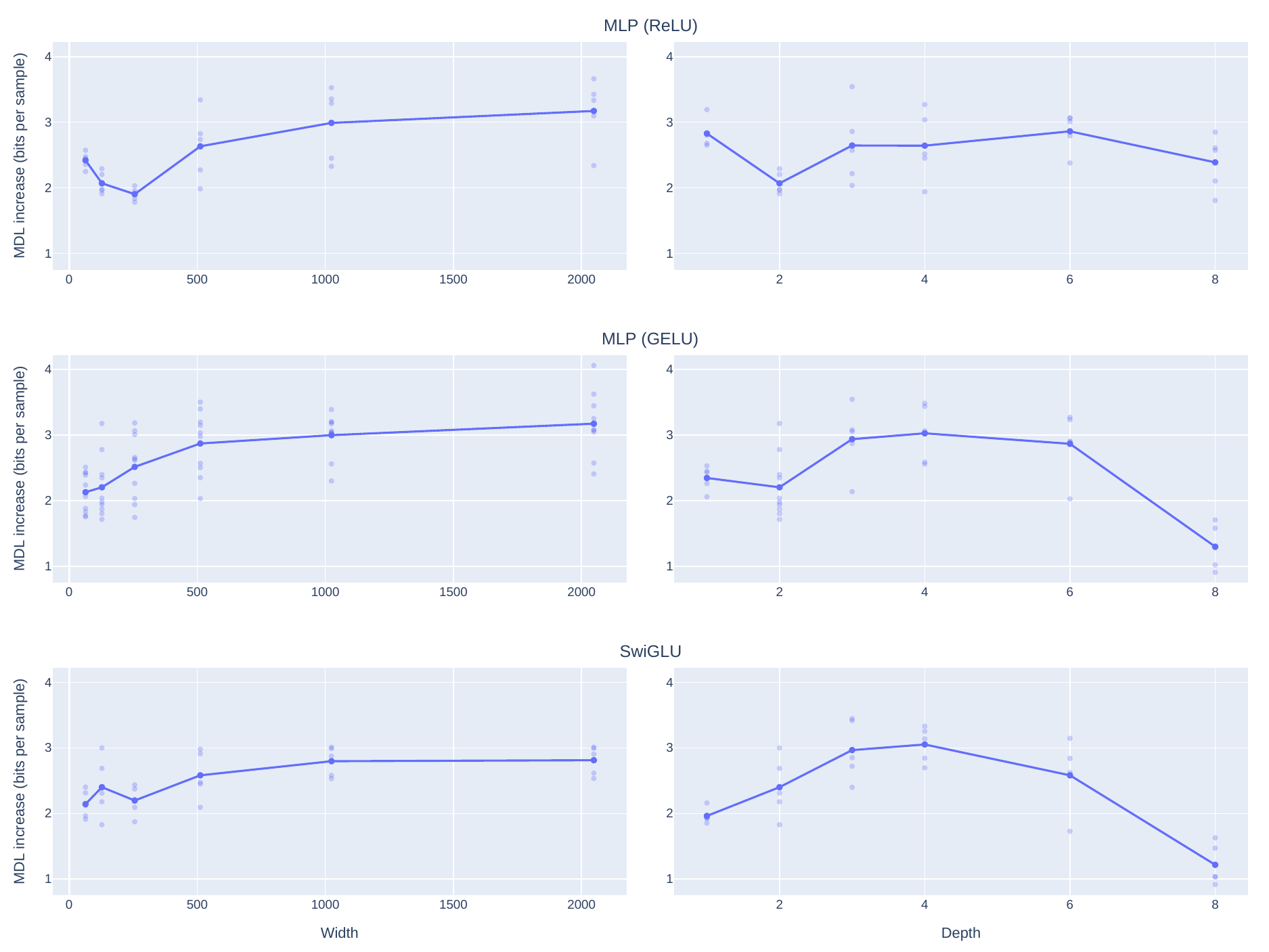}
    \caption{MDL over 5 random seeds for feedforward networks of various lengths and widths on linearly erased CIFAR-10. All models have a constant depth of 2 when width is varying and a constant width of 128 when depth is varying.}
    \label{fig:cifar10_ffn_mdl_linear}
\end{figure}

\begin{figure}[h]
    \centering
    \includegraphics[width=\textwidth]{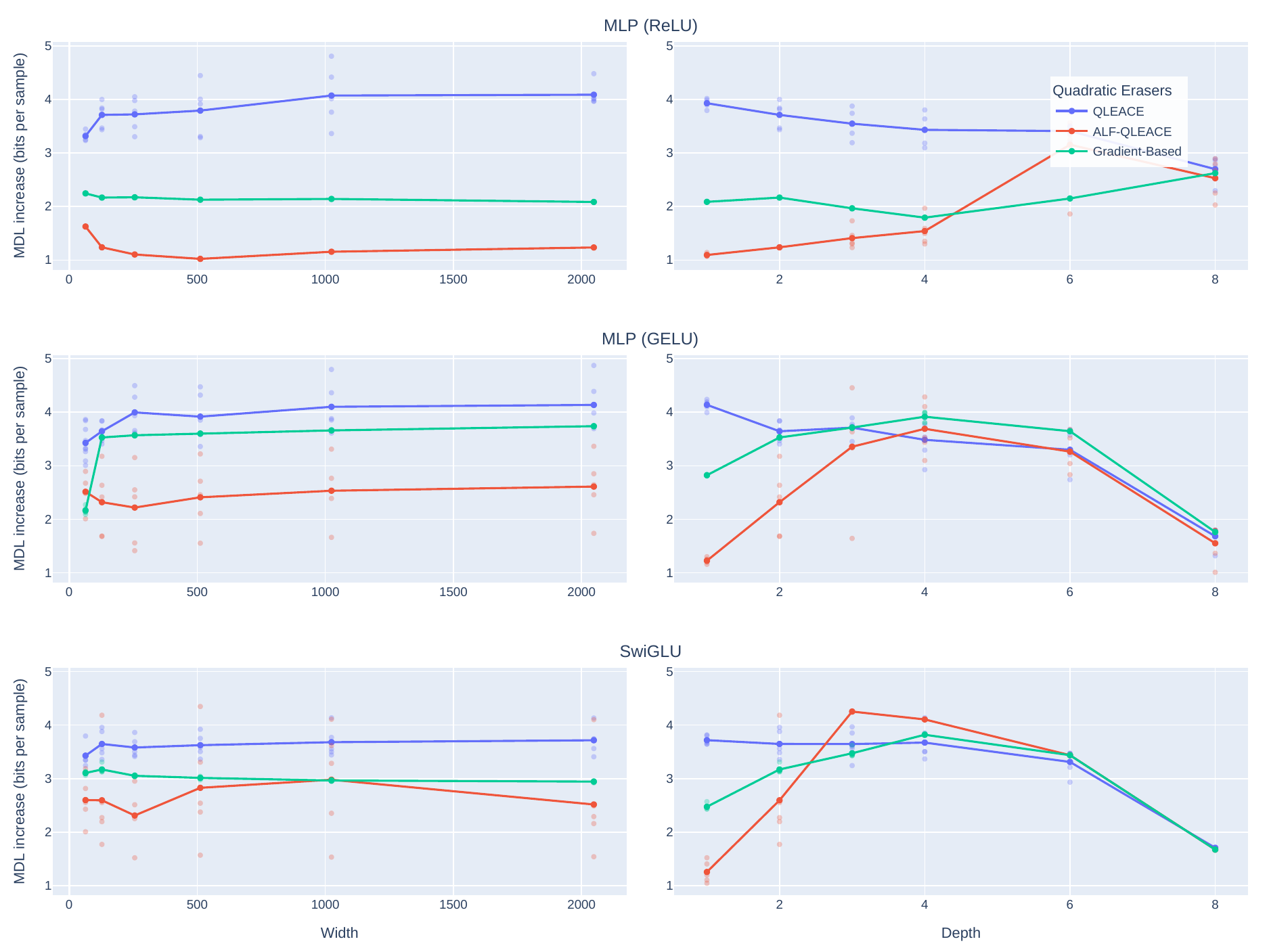}
    \caption{MDL over 5 random seeds for feedforward networks of various lengths and widths on the quadratically erased CIFAR-10 dataset. All models have a constant depth of 2 when width is varying and a constant width of 128 when depth is varying.}
    \label{fig:cifar10_ffn_mdl_quadratic}
\end{figure}

\begin{figure}[h]
    \centering
    \includegraphics[width=\textwidth]{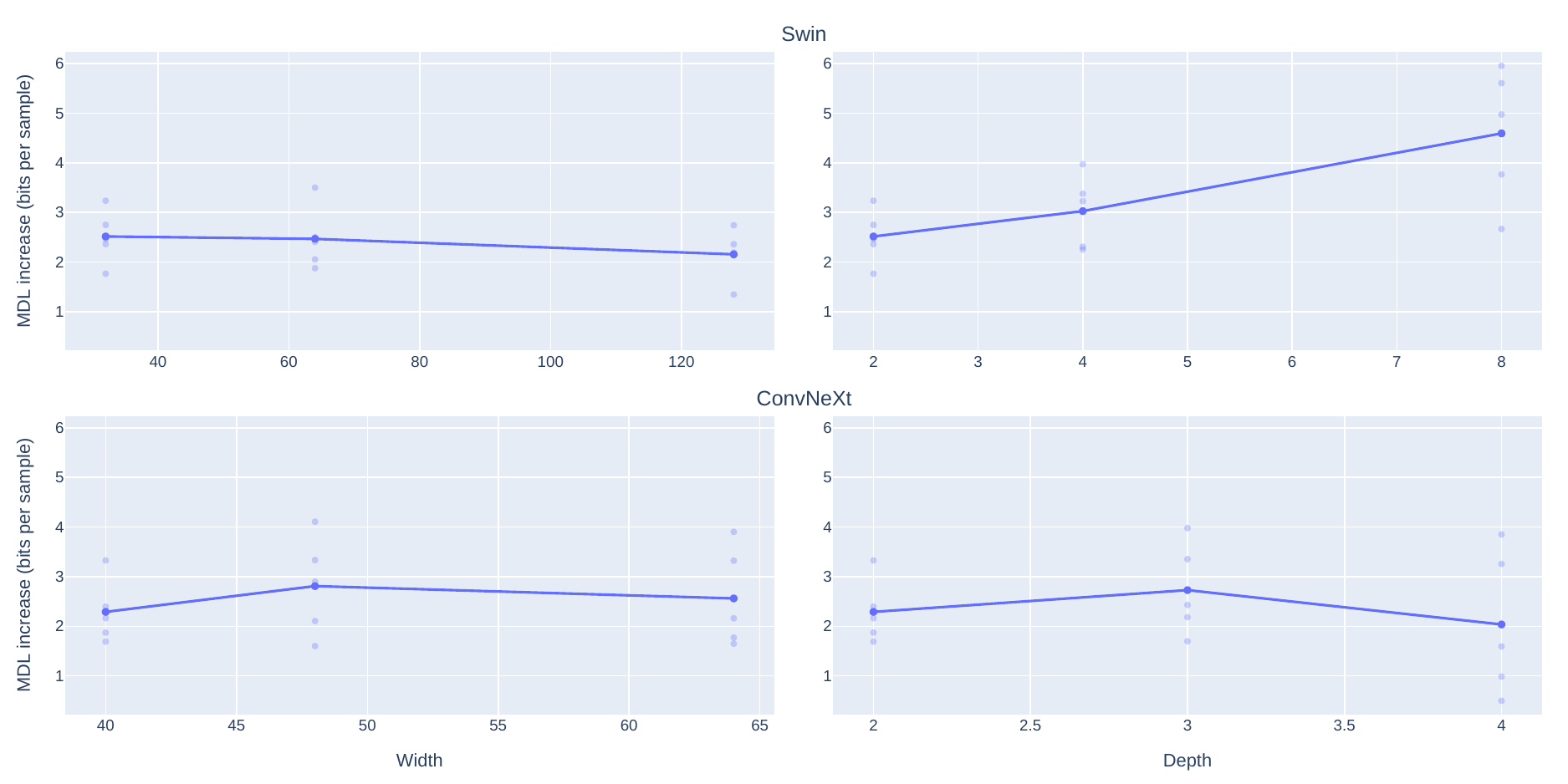}
    \caption{MDL over 5 random seeds for Swin Transformers and ConvNeXts of various lengths and widths on the linearly erased CIFAR-10 dataset.}
    \label{fig:cifar10_sota_mdl_linear}
\end{figure}

\begin{figure}[h]
    \centering
    \includegraphics[width=\textwidth]{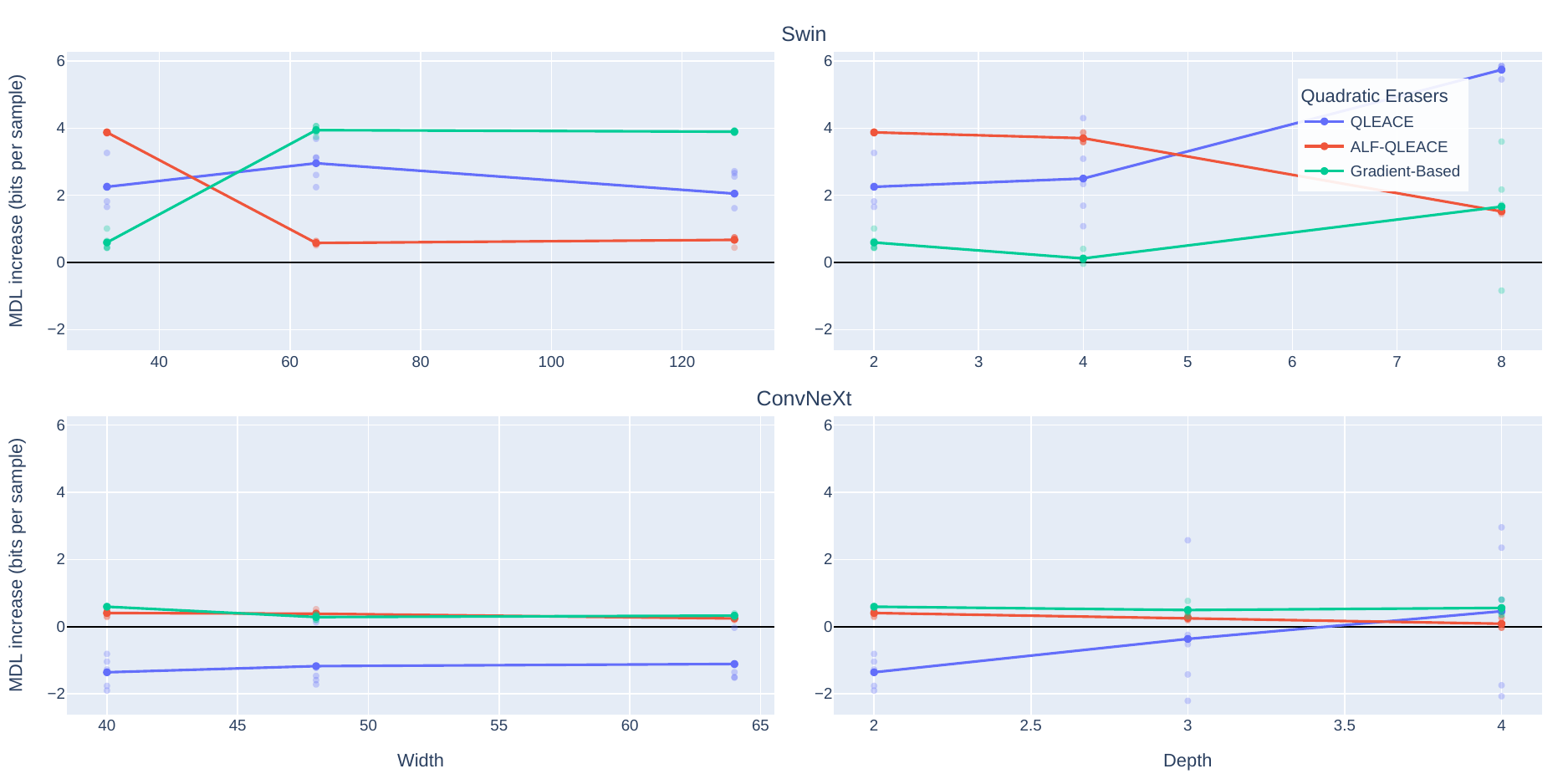}
    \caption{MDL over 5 random seeds for Swin Transformers and ConvNeXts of various lengths and widths on the quadratically erased CIFAR-10 dataset. The ConvNeXts exhibit backfiring, achieving lower mean MDLs on surgically quadratically erased data than on unerased data.}
    \label{fig:cifar10_sota_mdl_quad}
\end{figure}

\begin{figure}[h]
    \centering
    \includegraphics[width=\textwidth]{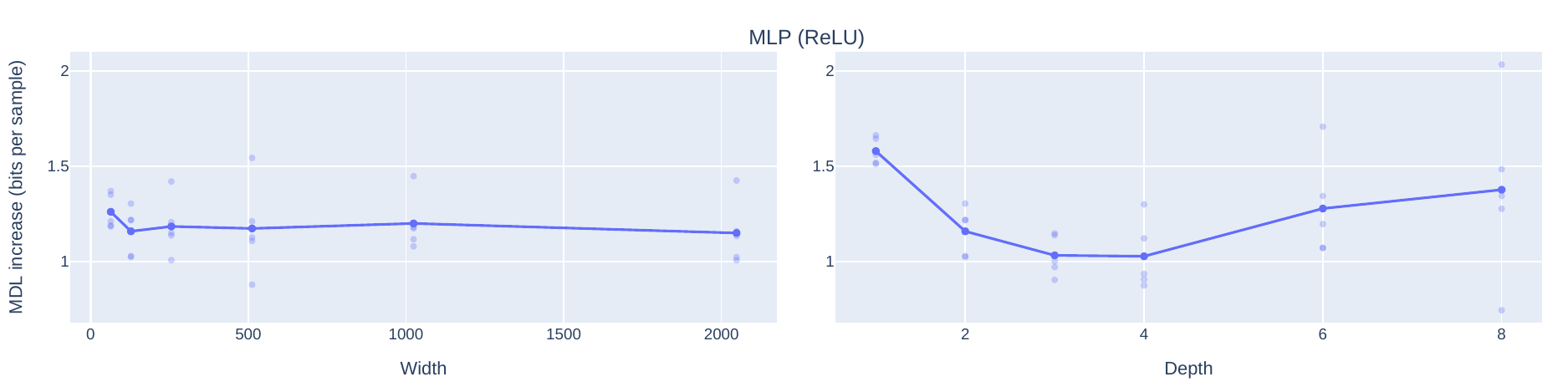}
    % \em{1}
    \includegraphics[width=\textwidth]{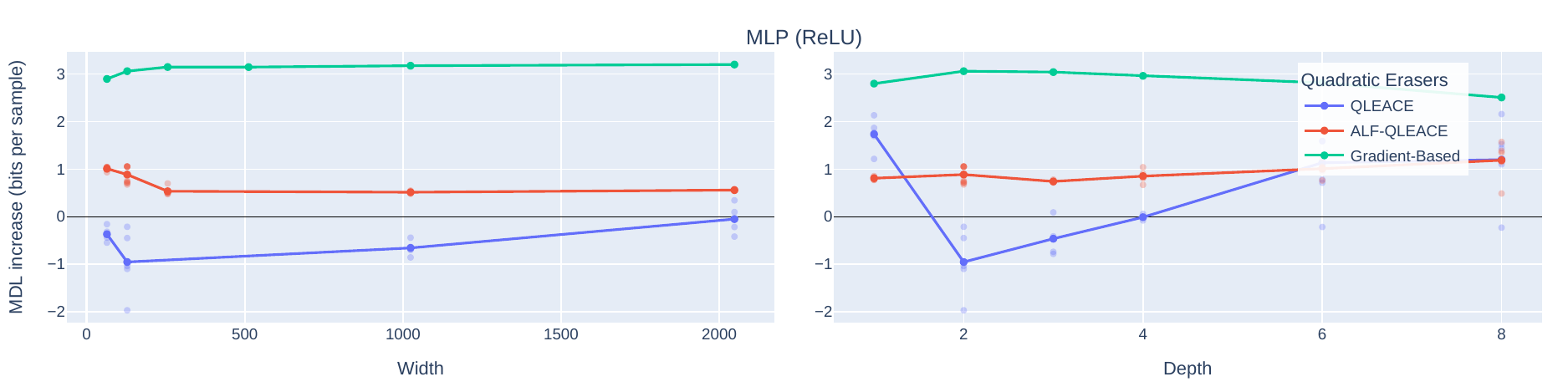}
    \includegraphics[width=\textwidth]{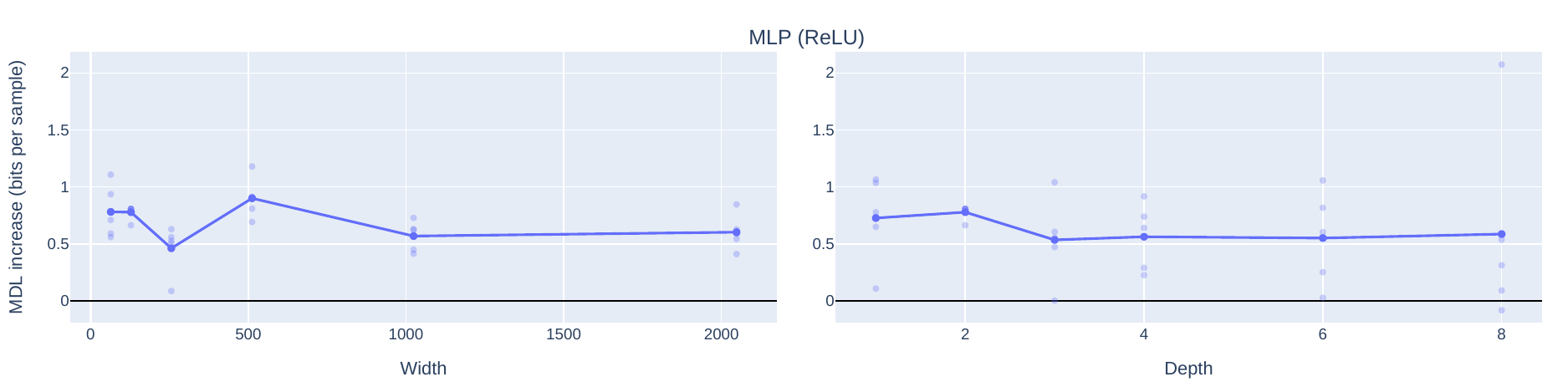}
    \includegraphics[width=\textwidth]{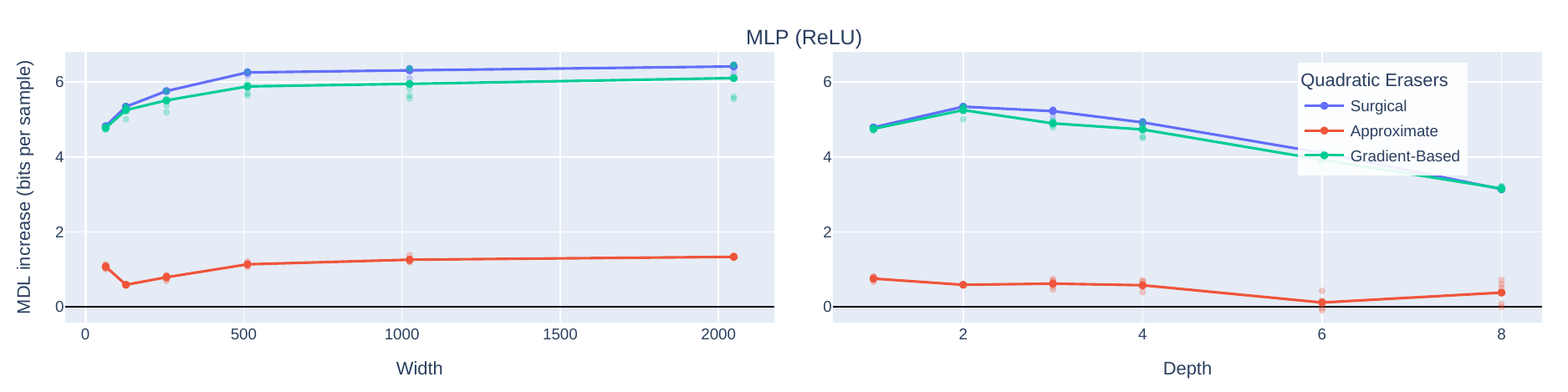}
    \caption{MDL over 5 random seeds for MLPs of various lengths and widths on the linearly and quadratically erased CIFARNet and SVHN datasets. The MLPs exhibit backfiring on the CIFARNet dataset for some model dimensions.}
    \label{fig:cifar10_sota_mdl_quad}
\end{figure}

\begin{figure}[h]
    \centering
    \includegraphics[width=\textwidth, trim=0 0 0 50pt, clip]{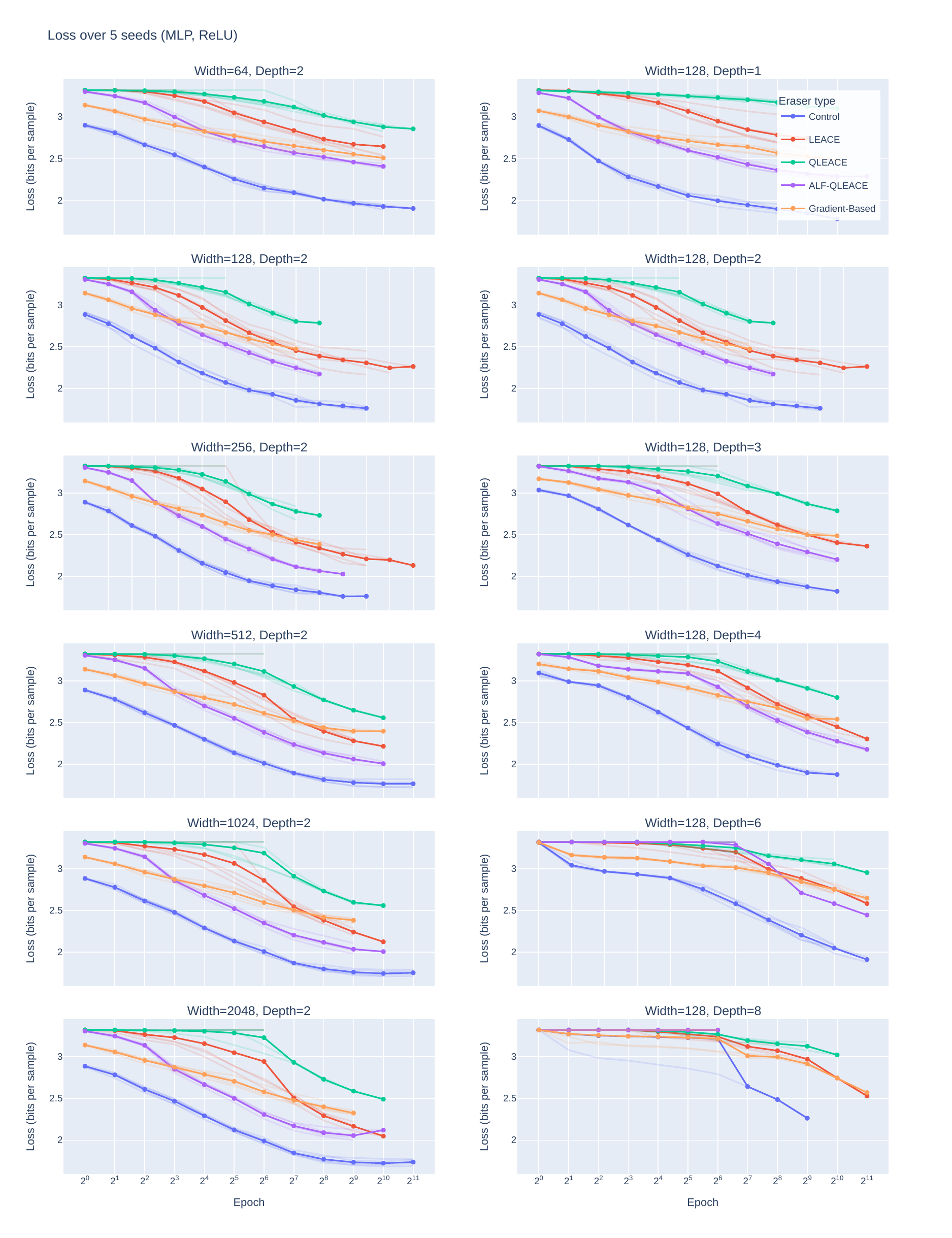}
    \caption{Cross-entropy loss over 5 random seeds for ReLU MLPs of various lengths and widths on the CIFAR-10 dataset.}
    \label{fig:mlp_relu_cifar10}
\end{figure}

\begin{figure}[h]
    \centering
    \includegraphics[width=\textwidth, trim=0 0 0 50pt, clip]{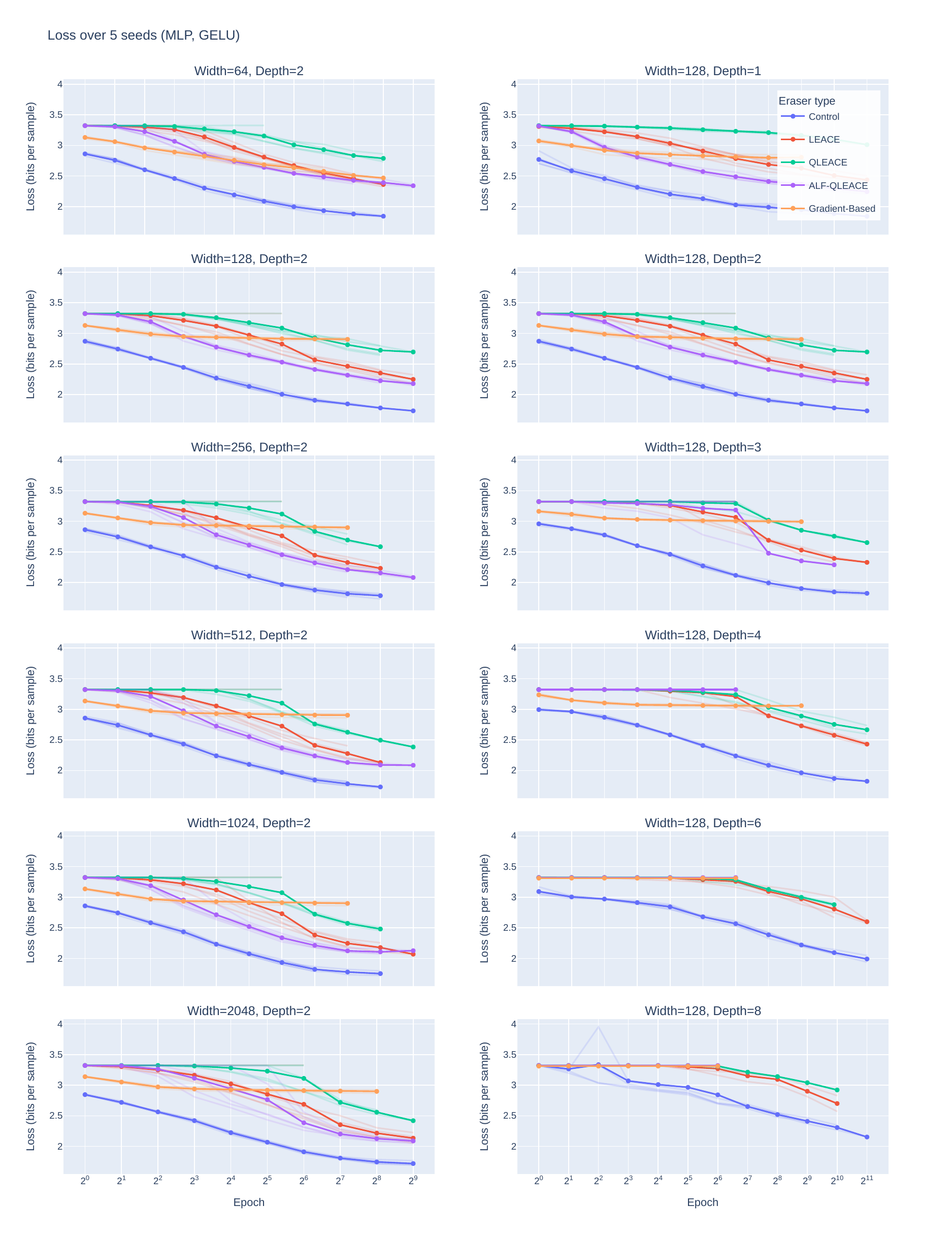}
    \caption{Cross-entropy loss over 5 random seeds for GELU MLPs of various lengths and widths on the CIFAR-10 dataset. MLPs with a varying depth have a constant width of 128 and MLPs with a varying width have a constant depth of 2.}
    \label{fig:mlp_gelu}
\end{figure}

\begin{figure}[h]
    \centering
    \includegraphics[width=\textwidth, trim=0 0 0 50pt, clip]{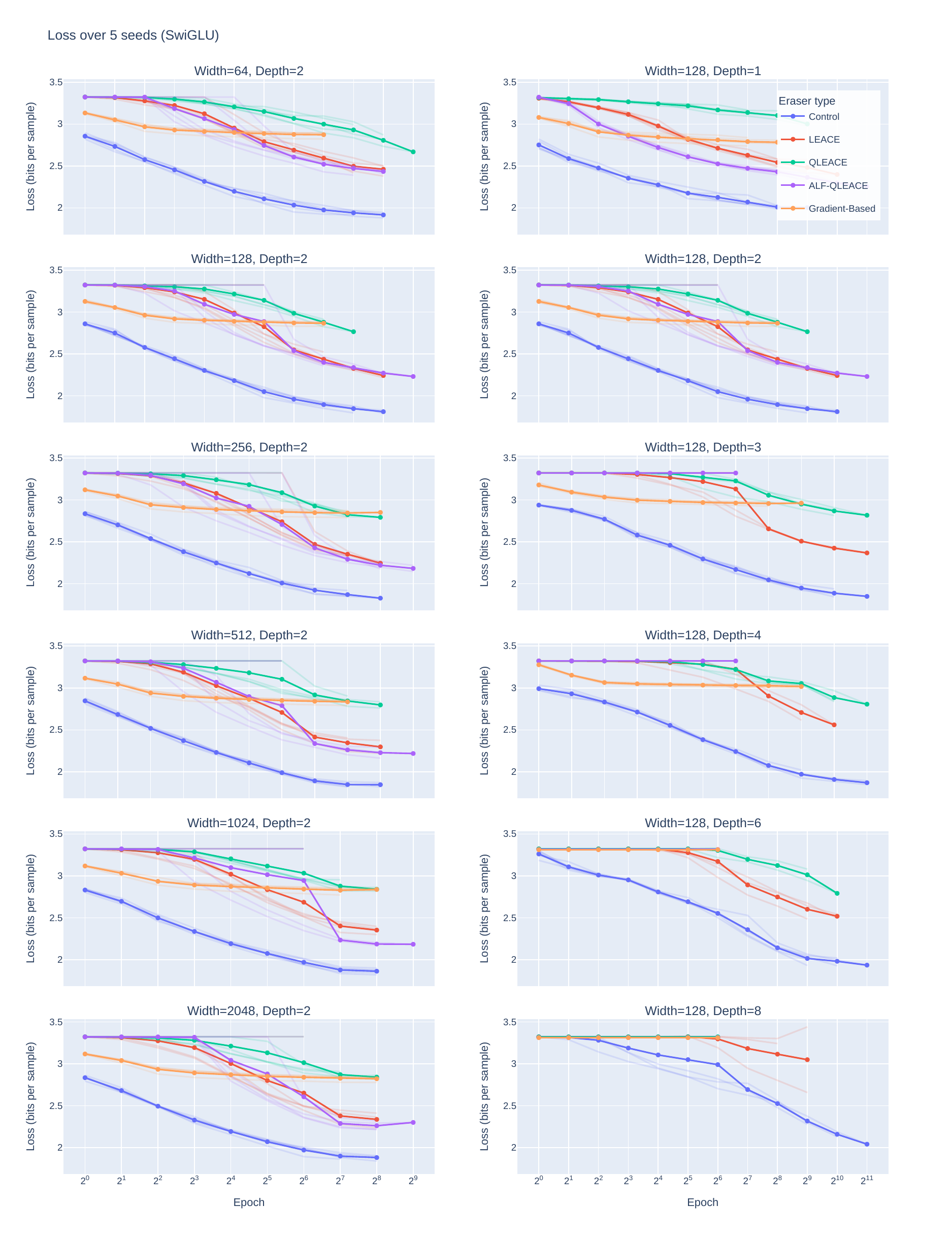}
    \caption{Cross-entropy loss over 5 random seeds for SwiGLUs of various lengths and widths on the CIFAR-10 dataset. GLUs with a varying depth have a constant width of 128 and GLUs with a varying width have a constant depth of 2.}
    \label{fig:mlp_swiglu}
\end{figure}

\begin{figure}[h]
    \centering
    \includegraphics[width=\textwidth, trim=0 0 0 50pt, clip]{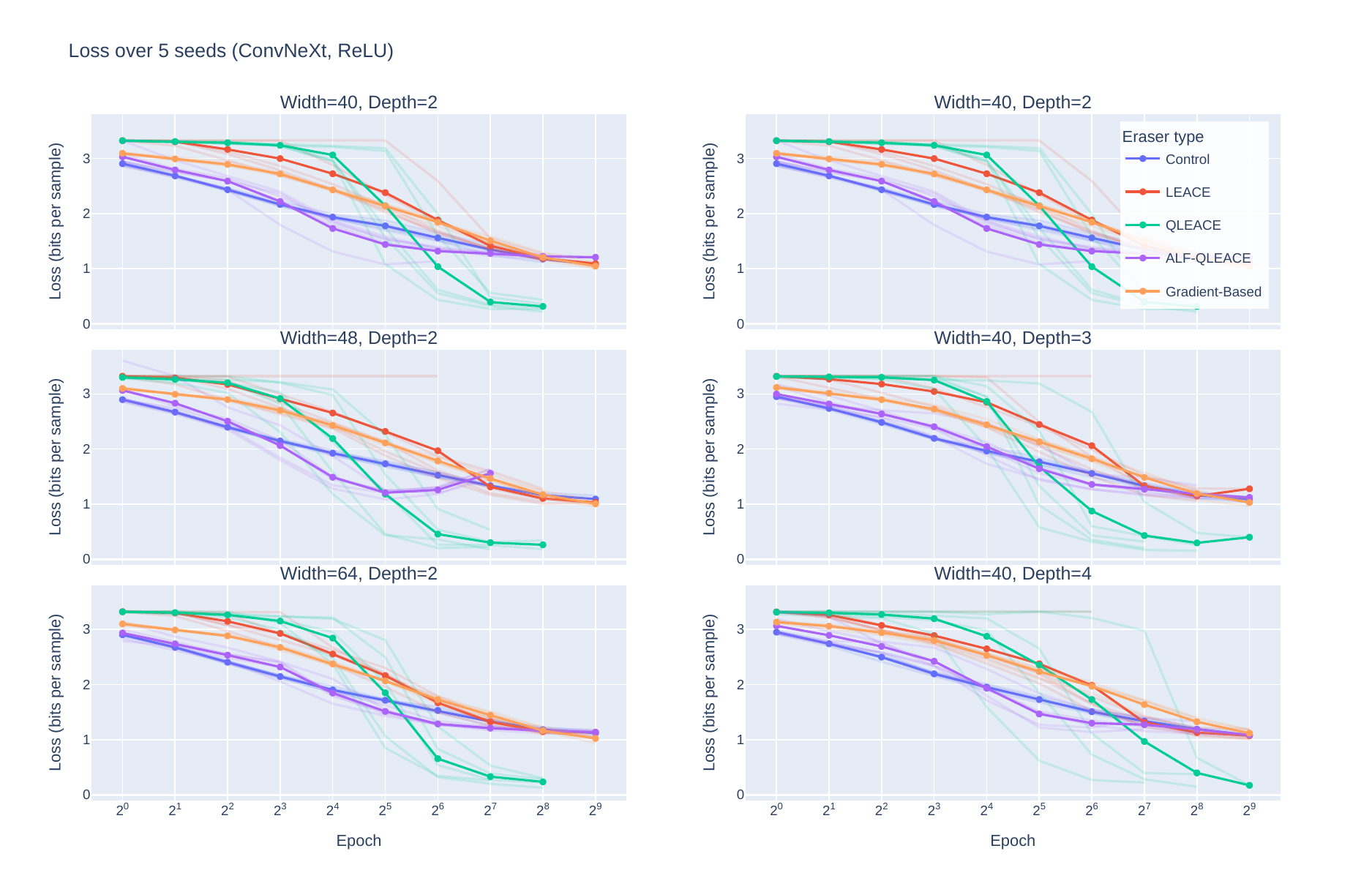}
    \caption{Cross-entropy loss over 5 random seeds for ConvNeXtV2s of various lengths and widths on the CIFAR-10 dataset.}
    \label{fig:convnext_loss}
\end{figure}

\begin{figure}[h]
    \centering
    \includegraphics[width=\textwidth, trim=0 0 0 50pt, clip]{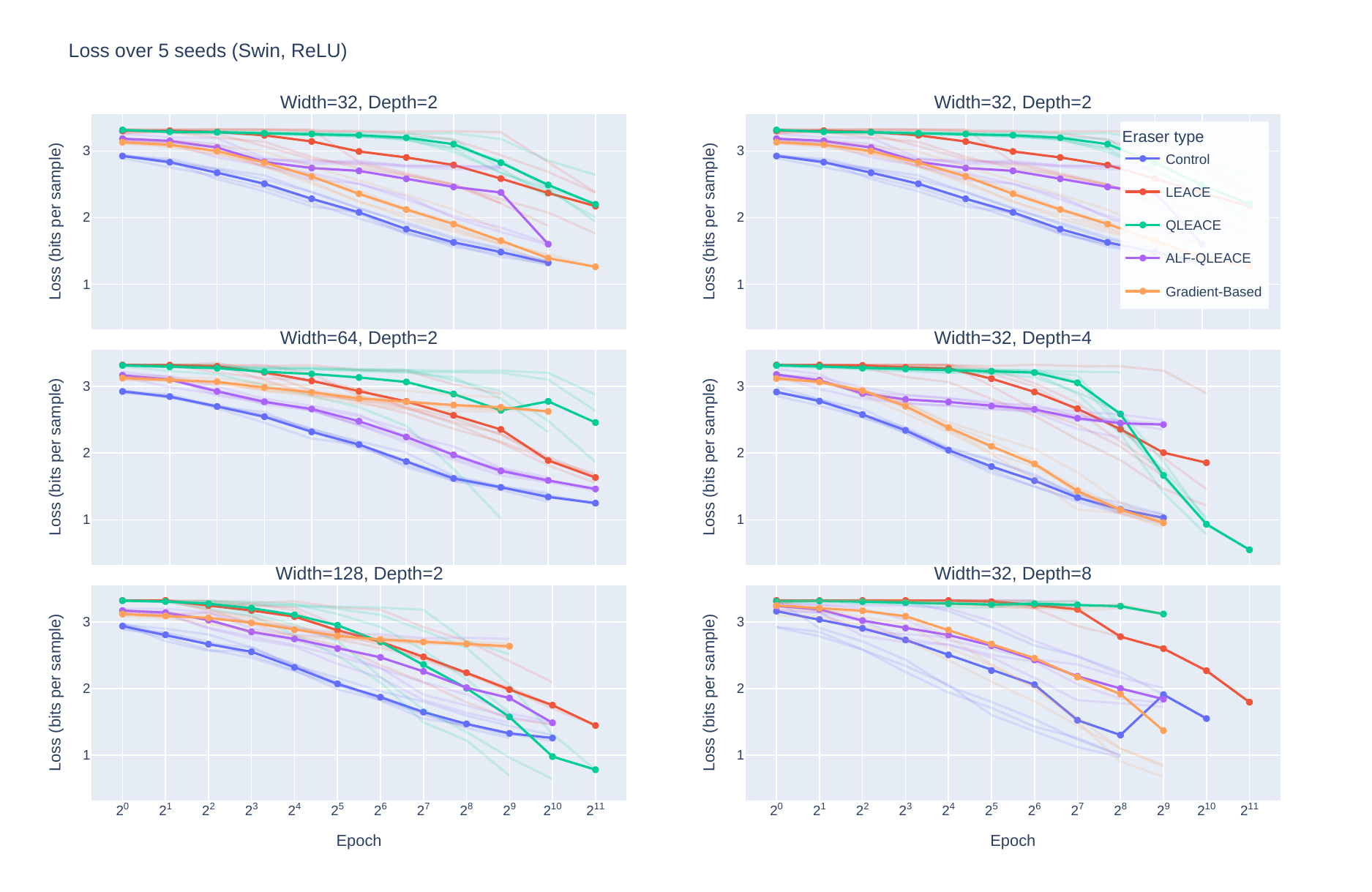}
    \caption{Cross-entropy loss over five random seeds for Swin transformers of various lengths and widths on the CIFAR-10 dataset.}
    \label{fig:swin_loss}
\end{figure}

\begin{figure}[h]
    \centering
    \includegraphics[width=\textwidth, trim=0 0 0 50pt, clip]{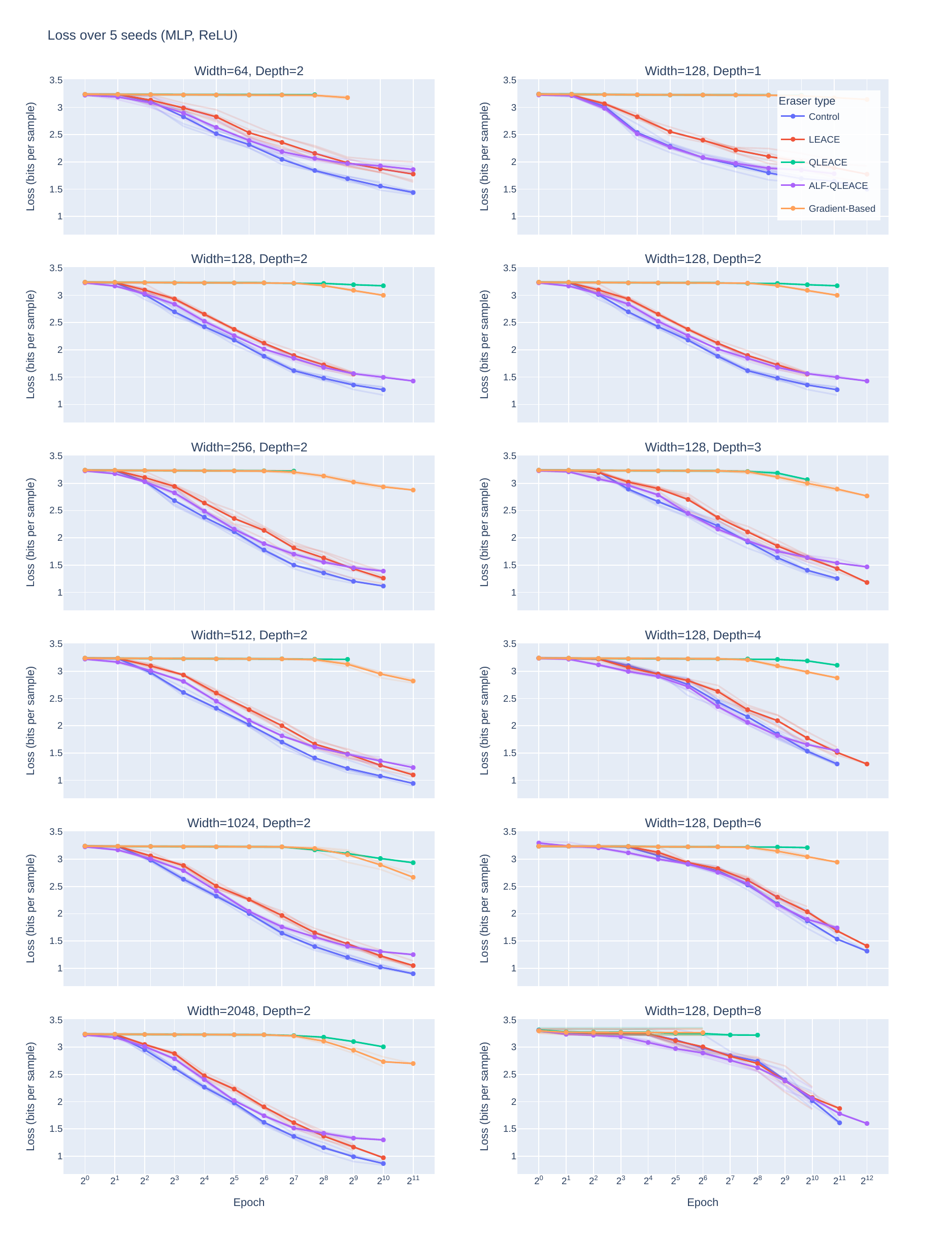}
    \caption{Cross-entropy loss over 5 random seeds for ReLU MLPs of various lengths and widths on the SVHN dataset.}
    \label{fig:mlp_relu_svhn}
\end{figure}

\begin{figure}[h]
    \centering
    \includegraphics[width=\textwidth, trim=0 0 0 50pt, clip]{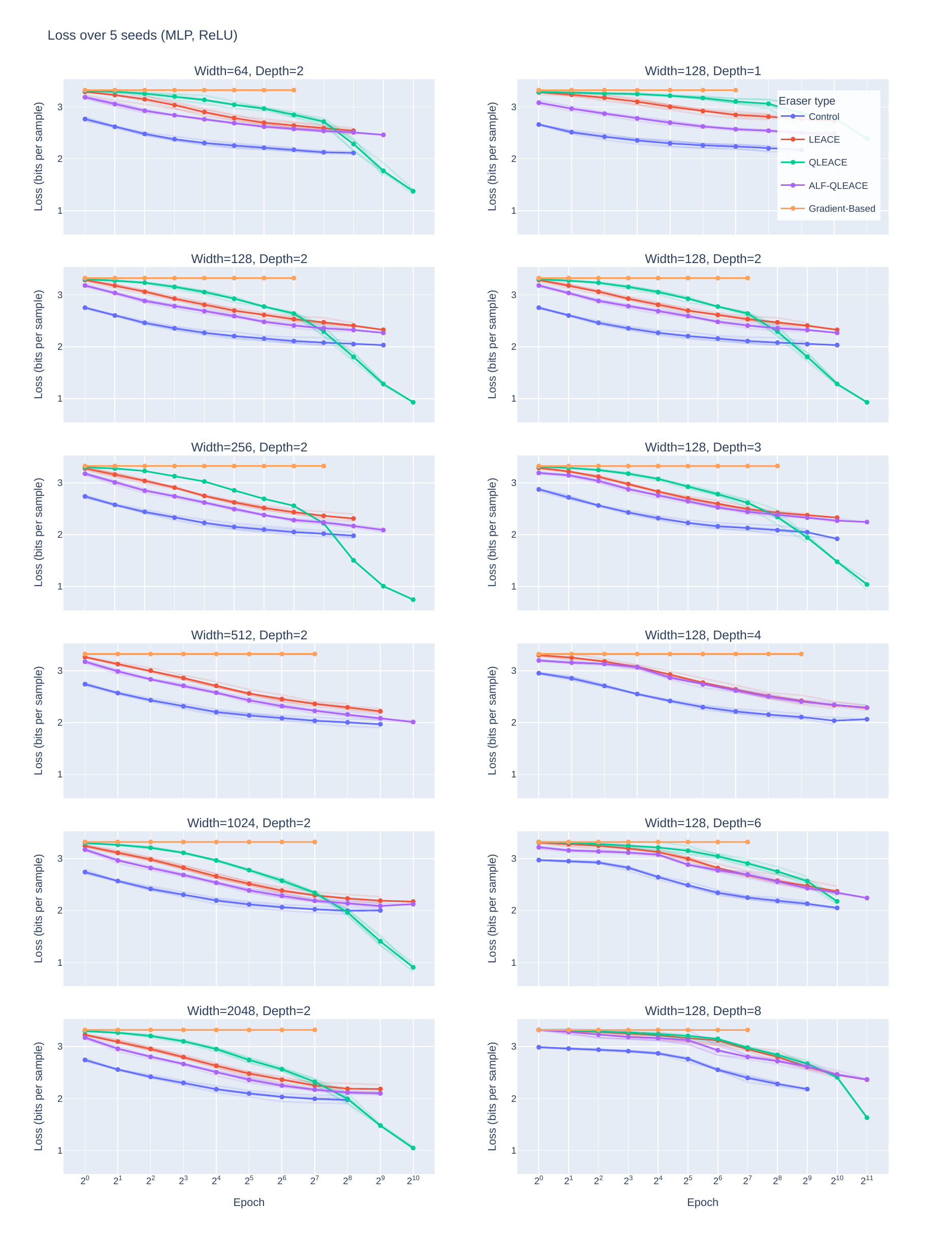}
    \caption{Cross-entropy loss over 5 random seeds for ReLU MLPs of various lengths and widths on the CIFARNet dataset.}
    \label{fig:mlp_relu_cifarnet}
\end{figure}

\begin{figure}[h]
    \includegraphics[trim=0 0 0 0, clip, width=\columnwidth]{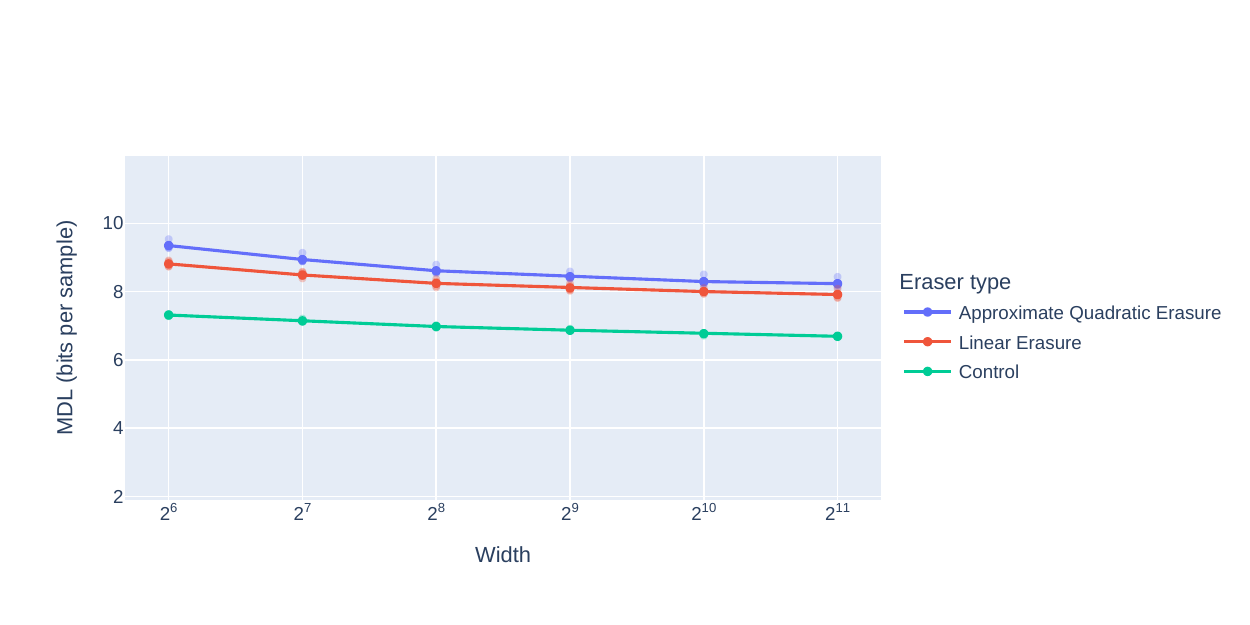}
    \vspace{1em}
    \centering
    \begin{tabular}{l|c|c|c|c|c}
    \hline
    Architecture & Control & LEACE & ALF-QLEACE & Gradient Quadratic & QLEACE \\
    % & & & Quadratic & Quadratic & Quadratic \\
    \hline
    MLP & 7.14 ± 0.05 & 8.49 ± 0.10 & 8.94 ± 0.14 & 10.02 ± 0.08 & \textbf{11.41} ± 0.00 \\
    LeNet & 5.62 ± 0.02 & 6.59 ± 0.28 & 6.83 ± 0.11 & \textbf{8.71} ± 0.09 & 8.68 ± 0.24 \\
    ConvNeXtV2 & 5.09 ± 0.09 & 5.20 ± 0.15 & 4.92 ± 0.05 & \textbf{5.37} ± 0.02 & 1.67 ± 0.20 \\
    Swin V2 & 8.18 ± 0.03 & 9.00 ± 0.13 & 8.67 ± 1.20 & \textbf{9.29} ± 0.06 & 9.08 ± 0.13 \\
    \hline
    \end{tabular}
    \caption{(Top) ReLU MLP MDL results on the CIFAR-10 dataset after applying z-score normalization. Normalization appreciably reduces MDL on linearly erased data, resulting in a lower mean MDL than for approximately quadratically erased data. (Bottom) Increases in MDL across architectures from erasing CIFAR-10 images where each (pixel, channel) coordinate has been normalized to zero-mean and unit-variance. All measurements are taken on base model sizes.}
    \label{tab:erasure_comparison}
\end{figure}

% \clearpage

\end{document}